\documentclass[letterpaper, 10 pt, conference]{ieeeconf}  

\IEEEoverridecommandlockouts                              

\usepackage{graphics} 
\usepackage{epsfig} 
\usepackage{mathptmx} 
\usepackage{times} 
\usepackage{amsmath} 
\usepackage{amssymb}  
\usepackage{bm}
\usepackage{mathtools}
\usepackage{xcolor}
\usepackage{upgreek}
\usepackage[capitalize]{cleveref}
\usepackage{tikz}
\usetikzlibrary{arrows.meta}
\usepackage{pgfplots}

\usepackage{calrsfs}
\DeclareMathAlphabet{\pazocal}{OMS}{zplm}{m}{n}
\DeclareMathOperator*{\argmax}{argmax}
\DeclareMathOperator*{\argmin}{argmin}

\newcommand{\appropto}{\mathrel{\vcenter{
  \offinterlineskip\halign{\hfil$##$\cr
    \propto\cr\noalign{\kern2pt}\sim\cr\noalign{\kern-2pt}}}}}

\newcommand\given[1][]{\:#1\vert\:}

\newtheorem{definition}{Definition}
\newtheorem{assumption}{Assumption}
\newtheorem{theorem}{Theorem}
\newtheorem{lemma}{Lemma}

\newcommand{\norm}[1]{\left\lVert#1\right\rVert}

\title{\LARGE \bf
Inverse Reinforcement Learning: A Control Lyapunov Approach
}

\author{Samuel Tesfazgi, Armin Lederer and Sandra Hirche
\thanks{All authors are with the Department of Electrical and Computer Engineering, Technical University of Munich, 80333 Munich, Germany
{\tt\small [samuel.tesfazgi,armin.lederer, hirche]@tum.de}}%
}

\begin{document}

\setlength{\textfloatsep}{8pt}
\setlength{\floatsep}{8pt}
\setlength{\abovedisplayskip}{5pt}
\setlength{\belowdisplayskip}{5pt}

\maketitle
\thispagestyle{empty}
\pagestyle{empty}

\begin{abstract}

Inferring the intent of an intelligent agent from demonstrations and subsequently predicting its behavior, is a critical task in many collaborative settings. A common approach to solve this problem is the framework of inverse reinforcement learning (IRL), where the observed agent, e.g., a human demonstrator, is assumed to behave according to an intrinsic cost function that reflects its intent and informs its control actions. 
In this work, we reformulate the IRL inference problem to learning control Lyapunov functions (CLF) from demonstrations by exploiting the inverse optimality property, which states that every CLF is also a meaningful value function.
Moreover, the derived CLF formulation directly guarantees stability of inferred control policies. We show the flexibility of our proposed method by learning from goal-directed movement demonstrations in a continuous environment.
\end{abstract}

\section{INTRODUCTION} \label{sec1}
Autonomous system are increasingly deployed in close proximity and conjunction with humans
, whether it is in healthcare robotics, semi-autonomous driving or manufacturing facilities. In many of these applications it is essential that the automation is able to effectively collaborate with the human partner. To this end, a model of the human's behavior is needed. This model can be used not only to predict the human partner, but also to teach machines complex behaviors, for which the design of a controller is cumbersome. This concept is called imitation learning or programming by demonstrations. Approaches based on this paradigm are advantageous, since preferences can be encoded through the demonstrations and the emulated behavior is easily interpretable \cite{Hussein17}. \looseness =-1

However, imitation learning is in general more difficult than just replicating the demonstrated behavior exactly, since the inferred control policy should ideally generalize well to unknown environments. Therefore, Abbeel and Ng \cite{Abbeel04} propose an indirect imitation learning approach called apprenticeship learning. Here, the agent is modelled to behave optimally with regards to an intrinsic cost function, which can be used to retrieve the agent's control policy. Hence, the imitation learning problem is reformulated to inferring the agent's cost function from demonstrations. This inference problem can be solved using inverse reinforcement learning (IRL), which is sometimes also referred to as inverse optimal control (IOC).\footnote{We consider IRL and IOC to be synonymous and use IRL subsequently.}
In \cite{Ziebart08}, Ziebert et al. expand the IRL framework by introducing the notion of suboptimalities in observed demonstrations through probabilistic models.
More recently, in \cite{Finn16}, neural network parameterizations are used to learn high dimensional nonlinear cost functions. However, the deployment of expressive function approximators has the drawback that it introduces model complexity to an already ill-posed inference problem.
In addition, most state-of-the-art IRL methods need to continuously solve the forward optimal control problem in order to evaluate the generated trajectories under the current parametrization of the cost function, which is computationally expensive in general. Furthermore, no guarantees regarding the convergence behavior of inferred control policies are provided. 

An alternative way of approaching the imitation learning problem are 
dynamical movement primitives (DMP). Here, the provided demonstrations of a goal-directed task are used to learn a dynamical system
capturing 
the observed behavior through its attractor landscape \cite{Ijspeert13}. 
For instance, asymptotically stable dynamical system models for describing robot reaching tasks are learned from demonstrations in \cite{Khansari-Zadeh2014}.
DMP methods generally allow to make strong statements regarding the convergence properties of inferred models, due to their grounding in dynamical systems theory.
Nevertheless, they lag behind the representational richness of IRL approaches. Since in DMP the agent and the system it is acting on are described as one dynamical system, no explicit model of the agent is retrieved and only pure imitation learning is possible.

In this work we propose a novel approach that uses the property of inverse optimality, which asserts that every Lyapunov function is a value function for some meaningful cost, to reformulate the IRL inference problem to learning CLFs from demonstrations. By utilizing tools of stochastic dynamical system theory we are able to provide formal guarantees for the stability of inferred control policies. Additionally, we propose non-parametric regression methods to learn the CLFs, therefore, resulting in flexible and data-driven human behavior models. The remainder of the paper is structured as follows: First the problem statement is introduced in \cref{sec2}, while the reformulation of the IRL problem is presented in \cref{sec4}. Finally, the evaluation of the proposed method follows in \cref{sec5}. \looseness=-1

\section{PROBLEM FORMULATION} \label{sec2}
Consider a discrete-time, control-affine system driven by the nonlinear dynamics\footnote{Notation: 
Lower/upper case bold symbols denote vectors/matrices, 
$\mathbb{R}_+$/$\mathbb{N}_+$ all real/integer positive 
numbers, $\bm{I}_n$ the $n\times n$ identity matrix,
$\|\cdot\|$ the Euclidean norm, and $\pazocal{C}(\pazocal{X})$ the set of continuous functions over a compact set $\pazocal{X}$. $\sqrt{\cdot}$ applied to a matrix means it is applied element-wise.}  
\begin{equation}
	\bm{x}_{t+1} = \bm{f}(\bm{x}_{t}) + \bm{g}(\bm{x}_t)\bm{u}_t, \label{eq:1}
\end{equation}
with continuous states $\bm{x} \in \pazocal{X} \subset \mathbb{R}^n$, where $\pazocal{X}$ is a compact set, and known $\bm{f}\colon \pazocal{X} \to \pazocal{X}$, $\bm{g}\colon \pazocal{X} \to \mathbb{R}^{n\times m}$ and initial condition $\bm{x}_0 \in \pazocal{X}$. The agent acting on the system is assumed to perform continuous control actions $\bm{u}\in \pazocal{U}(\bm{x}) \subseteq \mathbb{R}^m$, such that $\pazocal{U}(\bm{x}) = \left\{ \bm{u}\in\mathbb{R}^m \colon \bm{f}(\bm{x}) + \bm{g}(\bm{x})\bm{u} \in \pazocal{X} \right\}$. Since the agent performing the task is an expert, it is reasonable to assume that trajectories generated by the agent are bounded. Furthermore, \eqref{eq:1} is not a particular restrictive system class, since this structure holds for many mechanical systems. 

Following the optimality principle, the agent chooses its control actions~$\bm{u}_t$ according to a 
cost function with the goal of minimizing the accumulated stage costs $l\colon \pazocal{X}\!\!\times\pazocal{U}\! \to\! \mathbb{R}$ over time. Thereby, the agent employs the optimal policy \looseness=-1
\begin{subequations}
    \begin{align}
    \label{eq:disc_sum}
    \bm{\pi}^\ast(\bm{x}) = &\argmin_{\bm{\pi}\colon\pazocal{X}\to \pazocal{U}}~~ \sum\limits_{t=0}^\infty  \gamma^t l(\bm{x}_t,\bm{\pi}(\bm{x}_t)),\qquad \bm{x}_0=\bm{x} \\
    &\mathrm{such~that~} \bm{x}_{t+1} = \bm{f}(\bm{x}_{t}) + \bm{g}(\bm{x}_t)\bm{\pi}(\bm{x}_t),
\end{align}
\end{subequations}
where $\gamma\in \mathbb{R}_+$, $\gamma < 1$ is a discount factor.

While optimal control describes the problem of retrieving the 
optimal policy $\bm{\pi}^\ast$ given the known stage cost~$l$, 
IRL considers the inverse problem. Here, an expert agent demonstrates state-action trajectories $\{(\bm{x}_t, \bm{u}_t)\}_{t=0}^{T}$, with $T\in\mathbb{N}_+$. Subsequently, IRL infers the unknown cost function under which the associated optimal policy best explains the observations.
However, since the behavior of human agents does not exactly reflect the deterministic
optimal policy in practice, e.g., due to motor noise, the unaccounted variations in observed control actions are typically modelled as random perturbations \cite{Ziebart08}.
We consider perturbed policies of the form 
\begin{equation}
	\bm{u}_t = \bm{\tilde{\pi}}(\bm{x}_t) \coloneqq \bm{\pi}^\ast(\bm{x}_t) + \sqrt{\bm{\varSigma}(\bm{x}_t)}\bm{\omega}_t, \label{eq:4} 
\end{equation}
where the perturbations are described by the independent and identically distributed random variable $\bm{\omega}$ and state-dependent covariance $\bm{\varSigma}(\bm{x})\in\mathbb{R}^{m\times m}$. The perturbations $\bm{\omega}$ are generated according to a zero-mean truncated Normal distribution, where the truncation is chosen such that $\bm{\tilde{\pi}}(\bm{x}) = \bm{u}\in \pazocal{U}(\bm{x})$ remains true.
Thereby, the closed-loop dynamics 
\begin{equation}
	\bm{x}_{t+1} = \bar{\bm{f}}(\bm{x}_{t}) + \bm{g}(\bm{x}_t)\sqrt{\bm{\varSigma}(\bm{x}_t)}\bm{\omega}_t, \label{eq:5}
\end{equation}
with $\!\bar{\bm{f}}(\bm{x}_{t}) \!\coloneqq\!  \bm{f}(\bm{x}_{t}) \!+\! \bm{g}(\bm{x}_t)\bm{\pi}^\ast(\bm{x}_{t})$, become stochastic.

The following assumptions are made for the closed-loop dynamics and the perturbed policy:
\begin{assumption}
\label{as:unique}
The agent performs a goal-directed task, where the goal is defined by a unique target state ${\bm{x}^\ast\in \pazocal{X}}$.
\end{assumption}

\begin{assumption}
\label{as:perturb}
The variance in \eqref{eq:4} vanishes at the target; $\lim_{\bm{x}_t \to \bm{x}^\ast} \bm{\varSigma}(\bm{x}_t) = \bm{0}$, such that $\bm{\tilde{\pi}}(\bm{x}^\ast) = \bm{\pi}^\ast(\bm{x}^\ast)$.
\end{assumption}

Intuitively, \cref{as:perturb} states that the agent should act more deterministic close to the target $\bm{x}^\ast$, which is necessary for task completion according to \cref{as:unique}.
Since the target is unique, $\bm{x}^\ast$ coincides with the minimization of the agent's cost function.
Therefore, any agent that successfully performs the task has to act on the system such that it asymptotically converges to the desired final state $\bm{x}^\ast$. This has to hold true regardless of the stochasticity due to random perturbations in \eqref{eq:5}. In order to formalize this property, 
we introduce the following concept of stability.\looseness=-1
\begin{definition}[\cite{Li2013}]
\label{def:stab}
    A system \eqref{eq:5} has an 
    asymptotically stable equilibrium $\bm{x}^*$ on the set $\pazocal{X}$ in probability if
    \begin{enumerate}
        \item for all $\epsilon\! >\! 0$, $d \!>\! 0$, there exist  $\delta \!>\! 0$, $t_0\!\geq\! 0$ such that $\norm{\bm{x}_0 \!-\! \bm{x}^*} \!<\! \delta$ implies $P\left\{\norm{\bm{x}_t \!-\! \bm{x}^*} \!<\! d \right\} \geq 1\!-\!\epsilon,$ $\forall t \!\geq\! t_0.$
    \item 
    $P\left\{\lim_{t \to \infty}\norm{\bm{x}_t \!-\! \bm{x}^*} \!=\! 0\right\} \!=\! 1$ for all $\bm{x}_0\in\pazocal{X}$.
    \end{enumerate}
\end{definition}
Intuitively, \cref{def:stab} is the probabilistic analogue to the classical stability definition in the sense of Lyapunov~\cite{Khalil2002}. Therefore, we subsequently refer to this 
property as asymptotic stability.
Consequently, using \cref{as:unique} and 2, we can express the asymptotic minimization of the objective despite of stochasticity as follows.
\begin{assumption}
\label{as:stab}
Policy \eqref{eq:4} renders system \eqref{eq:1} asymptotically stable at the target state 
$\bm{x}^*$.
\end{assumption}

Based on this assumption, we consider the problem of determining 
the value function
\begin{equation}
	V^\ast(\bm{x}) =  \sum_{t=0}^\infty  \gamma^t l(\bm{x}_t, \bm{\pi}^*(\bm{x}_t)),\qquad \bm{x}_0=\bm{x} , \label{eq:2}
\end{equation}
which describes the minimum cost-to-go when starting in a state~$\bm{x}$ and following the optimal policy $\bm{\pi}^*$ thereafter. Due to the 
Bellman equality, the optimal policy $\bm{\pi}^*$ can be defined 
equivalently as \looseness=-1
\begin{align}
    \bm{\pi}^*(\cdot)=\argmin_{\bm{u}}\big[ l(\cdot,\bm{u})+\gamma V^\ast(\bm{f}(\cdot) + \bm{g}(\cdot)\bm{u}) \big].
    \label{eq:pi(V)}
\end{align}
Using this identity, the problem of estimating the stabilizing 
value function can now be formulated as a constrained
functional optimization problem. Here, the posterior of $V$ is maximized by evaluating the alignment of state-visitation probabilities, under the associated optimal policy $\bm{\pi}^*$, with observed trajectories $\{\tau_1,\ldots,\tau_N\}$, where $\tau_n=\{\bm{x}_{t}^{n}\}_{t=0}^{T}$ and $N\in\mathbb{N}_+$. In addition, \cref{as:stab} must hold true, which leads to the following formulation:
\begin{subequations}
\label{eq:funcopt}
    \begin{align}
    \label{eq:funcopt1}
    V^* = &\argmax_{V\in \pazocal{C}(\pazocal{X})} P\{ V\given \tau_1,\ldots,\tau_N \} \\
    &\mathrm{such~that~} \tilde{\bm{\pi}} \mathrm{~asymptotically~stabilizes~\eqref{eq:1}}.\label{eq:funcopt2}
\end{align}
\end{subequations}

From \eqref{eq:funcopt2} follows an additional constraint to the solution~$V^\ast$. This structural constraint holds over the whole solution space and not only at data points, therefore, providing information at unobserved states. Furthermore, any policy derived from $V^\ast$ is guaranteed to be stabilizing by design.

\section{Stability-Certified Inverse Reinforcement Learning} \label{sec4}
In order to infer the value function $V^\ast$ from training data 
$\tau_1,\ldots,\tau_N$, while considering the stability constraint
on the corresponding perturbed optimal policy $\tilde{\bm{\pi}}$, we exploit the 
inverse optimal relationship between value functions and CLFs. 
This allows us to transform the constraint \eqref{eq:funcopt2} into a 
Lyapunov-type constraint on the optimal value function in 
\cref{subsec:Lyap-Value}. By considering the optimal value function
as a control Lyapunov function, we can approximate the posterior maximization \eqref{eq:funcopt1} in a closed-form in \cref{subsec:Likopt}. 
In \cref{subsec:conkern}, the problem is finally cast
as a constrained kernel regression problem in order to efficiently solve
the problem using machine learning techniques.

\subsection{Lyapunov-Constrained Value Function Approximation}\label{subsec:Lyap-Value}
A practical method to ascertain the convergence property as introduced in \cref{def:stab}, without solving the underlying dynamical system equations, is by means of Lyapunov stability theory \cite{Li2013}.
Since the state space of the considered dynamics is bounded to a compact set, the following
relaxation regarding positive definiteness constraints of Lyapunov function candidates can be concluded here.
\begin{lemma}
\label{lem:stab}
Consider a stochastic system of the form \eqref{eq:5}, which generates
trajectories with states $\bm{x}_t$, such that $\bm{x}_0\in\pazocal{X}$
implies $\bm{x}_t\in\pazocal{X}$ almost surely for all $t\in\mathbb{N}_+$ 
and compact sets $\pazocal{X}$. 
If there exists a continuous $W\colon \pazocal{X} \to \mathbb{R}_{+,0}$, such that\looseness=-1
\begin{subequations}
\label{eq:stab_cond}
\begin{align}
	\mathbb{E}\left[W(\bm{x}_{t+1}) \given \bm{x}_t \right] - W(\bm{x}_t) &< 0, \enspace \forall \bm{x}_t\in \pazocal{X} \;\backslash\; \{\bm{x}^{\ast}\} \label{eq:9}\\
	\mathbb{E}\left[W(\bm{x}_{t+1}) \given \bm{x}_t \right] - W(\bm{x}_t) &= 0, \enspace \bm{x}_t= \bm{x}^{\ast}
\end{align}
\end{subequations}
then, the system with equilibrium point $\bm{x}^{\ast}$ is  asymptotically stable in the sense of \cref{def:stab}. 
\end{lemma}
\begin{proof}
We prove this lemma by showing that $\hat{W}=W-\min_{\bm{x}\in\pazocal{X}}W(\bm{x})$ is a Lyapunov function, which implies asymptotic stability in probability \cite{Li2013}. The decrease of $\hat{W}$ along system trajectories is ensured by \eqref{eq:9}, such that it remains to show that $\hat{W}$ is positive definite. This is identical to proving that $\bm{x}^*$ is the only minimizer of $W$ on $\pazocal{X}$. In order to show this 
let $\pazocal{M}=\{ \bm{x}\in\pazocal{X}: \min_{\bm{x}'\in\pazocal{X}}W(\bm{x}')=W(\bm{x}) \}$
be the set of minimizers of $W$. Assume that there
exists a $\hat{\bm{x}}\neq \bm{x}^*$, $\hat{\bm{x}}\in\pazocal{M}$. 
This implies that
\begin{align}
    \mathbb{E} [W(\bar{\bm{f}}(\bm{\hat{x}}) + \bm{g}(\bm{\hat{x}})\sqrt{\bm{\varSigma}(\bm{\hat{x}})}\bm{\omega})]-W(\hat{\bm{x}})\geq 0
\end{align}
by definition of the set $\pazocal{M}$. However, this is 
a contradiction of \eqref{eq:9}, such that $\bm{x}^*$ 
is the unique minimizer of $W$
. 
Therefore, $\hat{W}$ is 
positive definite, which concludes the proof.
\end{proof}

A major strength of Lyapunov stability theory is the existence of 
converse theorems, i.e., under the assumption of stability, a Lyapunov
function is guaranteed to exist. Since the Lyapunov-like function
$W$ considered in \cref{lem:stab} is merely a shifted Lyapunov function, this property extends. We exploit this together 
with the fact that every Lyapunov function is an optimal 
value function for some meaningful cost \cite{Freeman96}. This so called inverse optimality property
is employed to formulate the original problem \eqref{eq:funcopt}
as a Lyapunov-constrained optimization problem, which is guaranteed
to be feasible. This is shown in the following result.\looseness=-1
\begin{lemma}
The Lyapunov-constrained functional optimization problem
\begin{subequations}
	\label{eq:14}
\begin{align}
	\label{eq:14a}
	V^* = &\argmax_{V\in \pazocal{C}(\pazocal{X})} &&P\{ V\given \tau_1,\ldots,\tau_N \}, \\[3pt]
	\label{eq:14d}
	& \text{ s.t.}&&  \mathbb{E}\left[\Delta V\left(\bm{x}\right) \right]  \;\:< 0,\enspace	\forall \bm{x}\in\pazocal{X}\setminus \{\bm{x}^*\}	\\ 
	\label{eq:14e}
	& &&  \mathbb{E}\left[\Delta V\left(\bm{x}^*\right) \right] = 0,
\end{align}
\end{subequations}
with $\mathbb{E}\left[\Delta V\left(\bm{x}_{t}\right) \right] = \mathbb{E}\left[V(\bm{x}_{t+1}) \given \bm{x}_t \right] - V(\bm{x}_t)$ is feasible.
\end{lemma}
\begin{proof}
Since the perturbed optimal policy $\bm{\tilde{\pi}}$ asymptotically stabilizes system~\eqref{eq:1} in the sense of \cref{def:stab}, the converse Lyapunov theorem guarantees the
existence of a Lyapunov function satisfying constraints \eqref{eq:14d}-\eqref{eq:14e} \cite{Teel2014}. 
Moreover, every Lyapunov function also resembles an optimal 
value function for some meaningful cost $l$ 
\cite{Freeman96}, which makes the Lyapunov function a valid solution of \eqref{eq:14}. \looseness=-1
\end{proof}
Note that the realization of the stability constraint \eqref{eq:funcopt2} through the constraints \eqref{eq:14d} and \eqref{eq:14e} does not pose a restriction to 
the solution space, since every Lyapunov-like function is an
optimal value function for a continuum of stage costs with
different optimal control laws. Hence, the considered 
stability constraints still allow to infer different agent 
preferences modelled by the intrinsic costs~$l$.

\subsection{Closed-Form Likelihood Expression}\label{subsec:Likopt}

While the focus of the previous section lies on deriving a feasible
expression for the stability constraint, we deal with the problem
of maximizing the posterior \eqref{eq:funcopt1} in this section. 
In order to solve this 
problem, we follow a Bayesian approach, which directly leads to 
the proportional relationship
\begin{align}
    P\{V\given \tau_1,\ldots,\tau_N\}\propto P\{\tau_1,\ldots,\tau_N\given V  \} P\{V\},
\end{align}
where $P\{V\}$ denotes the prior probability distribution over
value functions $V$, which is a design choice. Since the
trajectories $\tau_1,\ldots,\tau_N$ are generated independently 
using the perturbed optimal policy $\tilde{\bm{\pi}}$, they are 
conditionally independent given the optimal policy $\bm{\pi}^\ast$. 
Due to \eqref{eq:pi(V)}, this implies 
\begin{align}
    P\{\tau_1,\ldots,\tau_N\given V  \} = P\{\tau_1\given V \}\cdots P\{\tau_N\given V  \}.
\end{align}
Similarly, observed states $\bm{x}_t^n$ along a trajectory $\tau_n$
are conditionally Markovian given the optimal policy $\bm{\pi}^\ast$. Thus, it follows by the same argument 
as before that 
\begin{align}
    P\{\tau_n\given V  \} = P\{\bm{x}_0^n\given \tau_n\} \prod_{t=1}^T
    P\{\bm{x}_t^n\given V, \bm{x}_{t-1}^n \},
\end{align}
where $P\{\bm{x}_0\given \tau_n\}=P\{\bm{x}_0\}$ is the prior initial state distribution. Since this prior is independent of $V$, we have
\begin{align}
    P\{\tau_n\given V  \} \propto  \prod_{t=1}^T
    P\{\bm{x}_t^n\given V, \bm{x}_{t-1}^n \}.
\end{align}
Because the agent behaves optimally, it follows that for each of these 
probabilities, the next state $\bm{x}_{t}^n$ generated by the 
policy $\bm{\tilde{\pi}}$ applied to the state $\bm{x}_{t-1}^n$ can 
be determined using~\eqref{eq:5}. Since the considered perturbations $\bm{\omega}$ to the optimal policy \eqref{eq:4} are truncated normally distributed, a closed form expression for the probabilities can be obtained. For improved readability, we assume $\omega\sim \pazocal{N}(\bm{0},\bm{I}_n)$, which is not particular restrictive in practice, since the deterministic part of \eqref{eq:4} dominates the perturbations in general. Accordingly, resulting in the following expression for the probabilities
\begin{align}
    &P\{\bm{x}_t^n\given V,\bm{x}_{t-1}^n \} =\\ 
    &\pazocal{N}(\bm{x}_t^n\given \bm{f}(\bm{x}_{t-1}^n)+\bm{g}(\bm{x}_{t-1}^n)\bm{\pi}^*(\bm{x}_{t-1}^n), \bm{g}^\intercal(\bm{x}_{t-1}^n)\bm{\varSigma}(\bm{x}_{t-1}^n)\bm{g}(\bm{x}_{t-1}^n))\nonumber,
\end{align}
where $\bm{\pi}^*$ is defined in \eqref{eq:pi(V)}. By combining 
all these equalities, we obtain the log-likelihood
\begin{align}
\label{eq:exact_loglik}
    \!\log(P\{V\!\given \tau_1,\ldots,\tau_N\})\!\propto\! \log(P\{V\})\!-\!\!\sum_{n\!=\!1}^N\sum_{t\!=\!1}^T (\bm{e}_t^n)^\intercal \bm{\Gamma}^{n}_{t}\bm{e}_t^n,\!
\end{align}
where \allowdisplaybreaks
\begin{align}
    \bm{e}_t^n&=\bm{x}_t^n-\bm{f}(\bm{x}_{t-1}^n)-\bm{g}(\bm{x}_{t-1}^n)\bm{\pi}^*(\bm{x}_{t-1}^n)\\
    \bm{\Gamma}_t^n&=\left(\bm{g}^\intercal(\bm{x}_{t-1}^n)\bm{\varSigma}(\bm{x}_{t-1}^n)\bm{g}(\bm{x}_{t-1}^n)\right)^{-1}.
    \label{eq:Gamma}
\end{align}
This function measures how 
well the perturbed optimal policy $\bm{\tilde{\pi}}$ approximates the observed demonstrations $\tau_1,\ldots,\tau_N$
of the agent. Thereby, the optimal value function $V^*$ is indirectly 
inferred, as it induces the unperturbed optimal policy $\bm{\pi}^*$ via \eqref{eq:pi(V)}, thus, influencing the loss in \eqref{eq:exact_loglik}. \looseness=-1

However, in order to maximize \eqref{eq:exact_loglik}  the optimal policy $\bm{\pi}^*$ is needed. A closed-form expression for $\bm{\pi}^*$ is generally difficult to obtain and is only available for some problems, e.g., linear systems with quadratic costs.
Although \eqref{eq:exact_loglik} can be practically employed 
in \eqref{eq:funcopt} by approximately solving the optimal control problem to determine the optimal policy $\bm{\pi}^*$, this
approach is computationally demanding in general \cite{Finn16}, particularly 
considering that the stability constraint 
needs to be enforced. 
Therefore, we make use of the fact that we
already know that $V$ is a Lyapunov-like function due to the constraints
\eqref{eq:14d} and \eqref{eq:14e}. Additionally, we know from the inverse optimality property and \cref{as:stab} that there exists a CLF, which is equivalent to $V^*$ \cite{Freeman96}.
This approach allows us
to construct closed-form control laws $\hat{\bm{\pi}}$ for many, more general
system classes, e.g., control-affine systems \cite{Byrnes1993, Bacciotti2001}. 
Hence, instead of employing the relationship between the optimal 
value function $V^*$ and the optimal policy $\bm{\pi}^*$, we 
propose to approximate the inference problem by exploiting 
the closed-form control law $\hat{\bm{\pi}}$ coming along with 
the CLF. 
Analogue to before, the  Lyapunov-like function $V$ is fitted by evaluating the incurred loss under the closed-form control law $\hat{\bm{\pi}}$.
This immediately leads to the approximation \looseness=-1
\begin{align}
\label{eq:approx_loglik}
    \!\log(P\{V\!\given \tau_1,\ldots,\tau_N\})\!\appropto\! \log(P\{V\})\!-\!\!\sum_{n\!=\!1}^N\sum_{t\!=\!1}^T (\hat{\bm{e}}_t^n)^\intercal\bm{\Gamma}^{n}_t\hat{\bm{e}}_t^n,\!
\end{align}
where 
\begin{align}
    \hat{\bm{e}}_t^n=\bm{x}_t^n-\bm{f}(\bm{x}_{t-1}^n)-\bm{g}(\bm{x}_{t-1}^n)\hat{\bm{\pi}}(\bm{x}_{t-1}^n).
    \label{eq:hate}
\end{align}
This log-likelihood still allows 
to infer $V$ indirectly by fitting a policy, 
merely substituting the optimal policy $\bm{\pi}^*$ by the 
closed-form policy $\hat{\bm{\pi}}$. Thereby, optimization 
with the approximate log-likelihood \eqref{eq:approx_loglik} 
yields a function $V$, which reflects the agent's preferences, since the associated policy
$\hat{\bm{\pi}}$ best possibly represents the observed 
demonstrations $\tau_1,\ldots,\tau_N$. Even though the obtained $V$ is not an optimal
value function in general, this formulation allows to 
encode arbitrary agent behaviors in principle, only
limited by the flexibility of the closed-form 
control policy. This flexibility is often sufficient 
to represent arbitrary training data. In the following theorem this 
is exemplarily shown for a time-discrete system with
control-affine structure~\eqref{eq:1}
and policy
\begin{align}\label{eq:clf-pol}
    \hat{\bm{\pi}}(\bm{x}) = -\beta \left[\nabla V(f(\bm{x}))g(\bm{x}) \right]^\intercal
\end{align}
where $\beta > 0$ and \eqref{eq:clf-pol} is taken from \cite{Bacciotti2001}.
\begin{theorem}
    Given a control-affine system and a data set with $N$ trajectories $\tau_n$ generated by an agent with stabilizing deterministic 
    policy $\hat{\bm{\pi}}$, i.e., $\bm{\varSigma}=\bm{0}$. If 
    $\bm{f}(\bm{x})\neq \bm{f}(\bm{x}')$ for all $\bm{x},\bm{x}'\in\pazocal{X}$, $\bm{x}\neq\bm{x}'$, there
    always exists a function $V$ such that the 
    closed-loop control law \eqref{eq:clf-pol} generates the 
    training data.
\end{theorem}
\begin{proof}
It it is straightforward to see that the control Lyapunov 
policy \eqref{eq:clf-pol} generates given trajectories $\tau_n$
if 
\begin{align}
    \bm{g}(\bm{x}_k^n)\bm{\pi}(\bm{x}_k^n)=-\beta \bm{g}(\bm{x}_k^n)\bm{g}^\intercal(\bm{x}_k^n)(\nabla V(\bm{f}(\bm{x}_k^n)))^\intercal
\end{align}
holds for all $t=0,\ldots,T-1$, $n=1,\ldots,N$. This immediately
gives explicit conditions for the gradient of $V$, e.g., if $\bm{g}(\bm{x}_k^n)$ has full rank, we have
\begin{align}
    \beta(\nabla V(\bm{f}(\bm{x}_k^n)))^\intercal=-\left(\bm{g}(\bm{x}_k^n)\bm{g}^\intercal(\bm{x}_k^n)\right)^{-1}\bm{g}(\bm{x}_k^n)\bm{\pi}(\bm{x}_k^n).
\end{align}
Therefore, this condition
 requires the existence of a function with specified 
gradient values, which can  be achieved with a continuous function
when its arguments are different. Since this is satisfied by assumption, i.e.,   $\bm{f}(\bm{x})\neq \bm{f}(\bm{x}')$ for all $\bm{x},\bm{x}'\in\pazocal{X}$, $\bm{x}\neq\bm{x}'$, a function $V$ 
generating trajectories $\tau^n$ using the policy \eqref{eq:clf-pol} 
is guaranteed to exist.
\end{proof}
Even though this finding demonstrates that the approximation is flexible enough to represent the training data, a possible optimality gap between $\bm{\pi}^*$ and $\hat{\bm{\pi}}$ remains. However, for uninformative priors in \eqref{eq:approx_loglik} and assuming the noiseless case $\tilde{\bm{\pi}} = \bm{\pi}^*$, it is possible to make the loss in \eqref{eq:hate} arbitrarily small, such that the the closed-form policy $\hat{\bm{\pi}}$ and optimal policy $\bm{\pi}^*$ agree on the training data. 
Since policies $\hat{\bm{\pi}}$ as defined in \eqref{eq:clf-pol} follow the gradient descent paradigm similarly as optimal policies, the gradient of a convex optimal value function exhibits the same direction as that of the learned CLF on the data. Hence, the level sets of the optimal value function $V^*$ and the CLF $V$ are similar in this case.

\subsection{Formulation as Constrained Kernel Regression Problem}\label{subsec:conkern}

While \eqref{eq:14} in combination with the approximate 
log-likelihood \eqref{eq:approx_loglik} is theoretically
appealing for inferring~$V$, it requires solving a functional
optimization problem, which is not tractable in practice. Therefore, expressive function approximators are needed. Due to the lack of structural knowledge regarding the value function $V$ of a human demonstrator, kernel-methods \cite{Hofmann2008} are suitable, since this class of machine learning techniques does not take a predefined structure, but instead constructs one solely using provided training data. To achieve this, they
rely on a kernel function \looseness=-1
\begin{equation}
	k \colon \pazocal{X} \times \pazocal{X} \to \mathbb{R}, \qquad (\bm{x}, \bm{x}') \mapsto k(\bm{x}, \bm{x}'),
\end{equation}
which defines a structure to the data by acting as a similarity measure between data points. 
Every kernel defines a so called 
reproducing kernel Hilbert space (RKHS)
\begin{align}
\label{eq:RKHS}
 \!\pazocal{H}_k\!=\!\! \Bigg\{ \!h\!\!:\hspace{0.13cm}&h(\cdot)\!=\!\sum\limits_{i=1}^N\!\alpha_i k(\cdot,\bm{z}_i),\left.\|h\|_k\!=\!\sum\limits_{i=1}^{\infty}\sum\limits_{j=1}^{\infty}\!\alpha_i\alpha_jk(\bm{z}_i,\bm{z}_j)\!<\!\infty \!\right\}\!\! ,   
\end{align} 
where $\alpha_i\in\mathbb{R}$ and $\bm{z}_i\in\pazocal{X}$. The RKHS comprises all the functions that can be approximated 
through the kernel. 

Since we cannot optimize over arbitrary functions, 
we restrict ourselves to a parameterization of continuous functions. Linear combinations
of universal kernels are particularly well-suited, since their spanned function space $\pazocal{H}_k$ 
is known to be dense in the continuous functions, thus, allowing to approximate continuous functions arbitrarily well \cite{Micchelli2006}.
Moreover, due to the analogy between certain kernel regression problems and Gaussian process regression \cite{Rasmussen2006}, the prior $P(V)$ can be 
intuitively defined, such that $\log(P(V))\propto -\|V\|_k$. 
This results in the constrained kernel regression problem \looseness=-1
\begin{subequations}
	\label{eq:kernelopt}
\begin{align}
	\label{eq:kernelopta}
	&\min_{V\in\pazocal{H}_k} &&  \sum_{n=1}^N\sum_{t=1}^T (\hat{\bm{e}}_t^n)^\intercal\bm{\Gamma}^{n}_t\hat{\bm{e}}_t^n +\lambda\|V\|_k^2, \\[3pt]
	\label{eq:kerneloptb}
	&\text{ s.t.} &&  \mathbb{E}\left[\Delta V\left(\bm{x}\right) \right]  < 0\quad	\text{for all}\enspace \bm{x}\in\pazocal{X}\setminus \{\bm{x}^*\},	\\
	\label{eq:kerneloptc}
	& &&  \mathbb{E}\left[\Delta V\left(\bm{x}^*\right) \right] = 0,
\end{align}
\end{subequations}
where $\lambda\in\mathbb{R}_+$ is a small constant. \looseness=-1

The stability constraint \eqref{eq:kerneloptb} remains a restriction
preventing the direct implementation of this optimization problem, as 
it must hold for an uncountable, infinite
set of states $\bm{x}$. This condition can only be resolved exactly in
simple problems, e.g., quadratic Lyapunov functions and linear 
systems. Therefore, we relax the stability condition and allow an
increase of the Lyapunov function along trajectories in a small neighborhood
$\pazocal{B}_\xi=\{\bm{x}\in\pazocal{X}: \|\bm{x}-\bm{x}^*\|\leq \xi \}$ of the equilibrium~$\bm{x}^*$, such that convergence of the 
closed-loop system trajectories to a neighborhood of the 
equilibrium is still guaranteed \cite{Umlauft2017a}. This allows us to approximate \eqref{eq:kernelopt} as 
a practically tractable optimization problem using a discretization
of the stability conditions, such that  
strong theoretical guarantees are retained as shown in the following theorem. \looseness=-1
\begin{theorem}\label{th:kernelopt}
Let $\pazocal{X}_\xi\!=\!\{\hat{\bm{x}}_1,\ldots,\hat{\bm{x}}_{N_\xi}\}$, $N_\xi\!\in\!\mathbb{N}_+$ be a discretization over $\pazocal{X}$ with 
grid constant $\xi\!\in\!\mathbb{R}_+$, i.e., $\xi\!=\!\max_{\bm{x}\in\pazocal{X}}\min_{\bm{x}'\in\pazocal{X}_\xi}\|\bm{x}\!-\!\bm{x}'\|$. Moreover, consider a continuous kernel $k$, continuous system dynamics $\bm{f}$ and $\bm{g}$ and let $L_{\Delta V}$ denote the Lipschitz constant of $\mathbb{E}\left[\Delta V\!\left(\bm{x}\right) \right]$. Then, any solution $V\!\in\!\pazocal{H}_k$ of the optimization problem 
\begin{subequations}
\label{eq:kernellagrange}
\begin{align}
	&\min_{V\in\pazocal{H}_k} &&  \sum_{n=1}^N\sum_{t=1}^T (\hat{\bm{e}}_t^n)^\intercal\bm{\Gamma}^{n}_t\hat{\bm{e}}_t^n +\lambda\|V\|_k^2, \\[3pt]
	& \text{ s.t.} &&  \mathbb{E}\left[\Delta V\left(\hat{\bm{x}}_i\right) \right]<-L_{\Delta V}\xi,\quad \forall \hat{\bm{x}}_i\in\pazocal{X}_\xi\setminus \{\hat{\bm{x}}^*\}\\
	&  &&\mathbb{E}\left[\Delta V\left(\hat{\bm{x}}^*\right) \right]\leq 0\\
	& &&\hat{\bm{x}}^*=\argmin_{\hat{\bm{x}}_i\in\pazocal{X}_\xi} V(\hat{\bm{x}}_i),
\end{align}
\end{subequations}
admits a representation of the form
\begin{align}
    V(\cdot) = \sum_{n=1}^N \alpha_n k(\bm{x}_n,\cdot)+\sum_{i=1}^{N_\xi}\alpha_{N+i}k(\hat{\bm{x}}_i,\cdot), 
    \label{eq:kernelsolution}
\end{align}
and satisfies condition \eqref{eq:14d} for system \eqref{eq:5} on all $\pazocal{X}\setminus \pazocal{B}_\xi$.
\end{theorem}
\begin{proof}
The constrained optimization problem \eqref{eq:kernellagrange} can be transformed into an unconstrained one by using Lagrange multipliers. Since the regularizer $\|V\|_k^2$ is strictly increasing on $[0,\infty[$, the obtained unconstrained regularized risk functional conforms to the requirements of the generalized representer theorem, which guarantees that a solution of the form \eqref{eq:kernelsolution} is admitted \cite{Scholkopf2001}. This proves 
the first part of the theorem. The second part follows from 
continuity of the kernel $k$, such that all functions of the
form \eqref{eq:kernelsolution} are continuous. Together with the continuity of $\bm{f}$ and $\bm{g}$ follows that $\Delta V$ is continuous, which immediately
implies the existence of a Lipschitz constant $L_{\Delta V}$
on $\pazocal{X}$. Therefore, tightening the stability constraint
\eqref{eq:14d} by $L_{\Delta V}\xi$ on the grid $\pazocal{X}_{\xi}\setminus\{\hat{\bm{x}}^*\}$ ensures that it is 
satisfied for all $\bm{x}\in\pazocal{X}\setminus\pazocal{B}_{\xi}$,
which concludes the proof. \looseness=-1
\end{proof}
In order to transform the stability conditions into tractable constraints, we adopt a discretization approach in \cref{th:kernelopt}, which is a commonly used method in the context 
of numerical analysis of non-parametric Lyapunov functions~\cite{Berkenkamp2016a, Lederer2019b}. It can be seen clearly
that this approach requires trading-off computational
complexity and flexibility of the feasible solutions, since small
values of $\xi$ result in a slight constraint tightening, 
but in turn cause a high computational complexity. This can be
exploited by employing multi-resolution grids, 
in which the grid constant is adapted to the training data 
and the dynamics, such that high flexibility is provided where 
required, while unnecessary computational effort is avoided. 
In order to further reduce the conservatism of the constraint
tightening, multi-resolution grids can be combined with local
Lipschitz constants, which can be efficiently computed 
for kernel methods using GPUs  \cite{Lederer2020d}.
\vspace{-0.025cm}
\section{Evaluation} \label{sec5}

For the evaluation of the proposed approach, we examine two scenarios. First, the principle capacity of the CLF to approximate an optimal value function is shown, by considering a linear-quadratic problem for which ground truth information is known. Secondly, the flexibility of the method is demonstrated by learning from human demonstrations in a goal-directed movement task. \looseness=-1

In the first problem, we consider a system governed by the linear dynamics $\bm{x}_{t+1} = \bm{A}\bm{x}_{t} + \bm{B}\bm{u}_t$ and define a quadratic running cost $c(\bm{x}_{t}, \bm{u}_t) = \bm{x}_{t}^\intercal\bm{Q}\bm{x}_{t} + \bm{u}_{t}^\intercal\bm{R}\bm{u}_{t}$, where
\begin{equation}
    \bm{A} \!= \!
    \begin{bmatrix}
    1  & 0.1 \\
    0  & 0.9 
\end{bmatrix}\!\!, \;
    \bm{B} \!= \!
    \begin{bmatrix}
    1  & 0 \\
    0  & 1 
\end{bmatrix}\!\!, \;
    \bm{Q} \!= \!
    \begin{bmatrix}
    1  & 0 \\
    0  & 0.5 
\end{bmatrix}\!\!, \;
    \bm{R} \!= \!
    \begin{bmatrix}
    15  & 0 \\
    0  & 15 
\end{bmatrix}. \nonumber
\end{equation}
Using the LQR it is straightforward to compute the optimal feedback control law and optimal value function, which takes a quadratic form here. To learn the CLF we observe the one step trajectories generated by the deterministic optimal feedback control, i.e., $\bm{\varSigma}=\bm{0}$, at $N=120$ points. 
The state space is bounded in $\pazocal{X} \in [-5, 5]^2$ and an $11\times11$ grid spanning equidistantly over the state space is used to ensure the stability constraint defined in \eqref{eq:kernellagrange}. The regularization factor is set to $\lambda = 1/\bar{N}$, where $\bar{N} = NT + N_{\xi}$ equals the total number of observed data points. The optimization is initialized with $\bm{\alpha} = \bm{0}$ and the control law is given by \eqref{eq:clf-pol} with $\beta = 0.2$.
\begin{figure}
    \centering
    \includegraphics{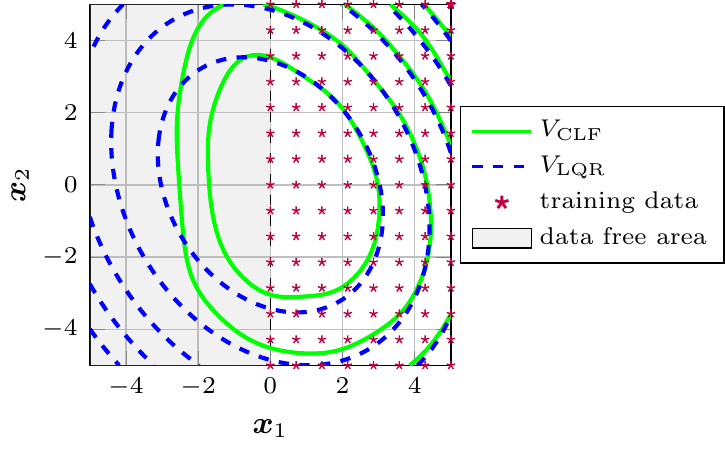}
    \vspace{-0.4cm}
    \caption{Contours of the ground truth value function illustrated in blue, dashed lines and the learned CLF given in green. The purple stars depict the training data, while the greyed out parts shows the data free area. \looseness=-1}
    \label{fig:LQR_comparison}
\end{figure}
In \cref{fig:LQR_comparison} the contour plots for the ground truth value function and the learned CLF are depicted. It can clearly be seen that both functions agree in the shape of their level curves in areas of the state space, where training data is provided. Thus, the learned Lyapunov-like function constitutes an appropriate approximation for the optimal value function. However, the approximation does not take a quadratic form in the data free area, since the kernel-based approach does not assume a predefined structure and requires data to build one. 
Therefore, our proposed method remains flexible in approximating the value function and does not limit the expressiveness of resulting control policies. We demonstrate this flexibility in the following learning from human demonstrations task. \looseness=-1

\begin{figure}
    \centering
    \includegraphics{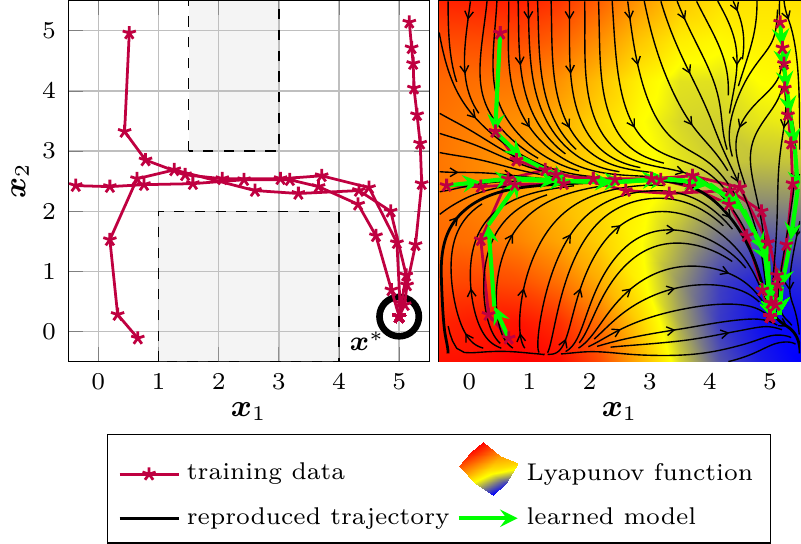}
    \vspace{-0.2cm}
    \caption{(Left) Depiction of the goal-directed movement task environment. The goal position $\bm{x}^\ast$ is indicated by the bold circle on the bottom right. The grey, transparent rectangles visualize areas associated with high costs. With purple lines the human demonstrations are depicted. (Right) Visualization of the kernel regression result for the goal-directed movement task. The learned Lyapunov function is visualized in form of a projected surface plot and as streamlines (black) along the gradient direction. The simulated trajectories under the inferred control law $\hat{\bm{\pi}}$ are given in green. }
    \label{fig:Env_CLF}
\end{figure}

The task is performed in a two-dimensional, continuous space, where the demonstrator attempts to reach a target position efficiently, whilst avoiding areas associated with high penalties. The environment and human generated demonstrations are depicted in \cref{fig:Env_CLF} on the left. Generally, this setup can be thought of as an abstraction of typical
programming by demonstration tasks, i.e., human reaching motions or lane tracking during driving.
The considered system is driven by the nonlinear dynamics
\begin{align}
    \bm{f}(\bm{x}) = 
    \begin{bmatrix}
        x_{1} \left(1 + s(x_{1}) \right)  \\
        {x}_{2} \left(1 + s({x}_{2})\cos^2{(\frac{\pi}{5}{x}_{2} } ) \right) 
    \end{bmatrix}
\end{align}
where  $s(x) = (2/(1+\exp(x-2.5)) - 1)\sin^2{(\pi{x}/5} )$ and $\bm{g}(\bm{x})=(2-\cos^2{(\pi{x_2}/5})\sin^2{(\pi{x_1}/5}))\bm{I}_2$, $\bm{x}=[x_1\ x_2]^\intercal$.

The state space is bounded in $\pazocal{X} \in [-0.5, 5.5]^2$ and a $10\times10$ grid spanning equidistantly over the state space is used. The number of observed trajectories is $N=4$ and we sample $T=11$ observations from each trajectory. 
For the computation of the expectation in \eqref{eq:kernellagrange} we draw and average over five points generated by a Gaussian distribution with standard deviation $\sigma = 0.05$, centered around the deterministic $V$. The regularization factor $\lambda$, the initialization of $\bm{\alpha}$ and the control factor $\beta$ stay the same as before.

From the results in \cref{fig:Env_CLF} it can clearly be seen that the proposed method is capable of simultaneously learning a Lyapunov function and an accompanying control law that drives any point in the bounded state space to the desired equilibrium point, i.e., the system is asymptotically stable. 
When comparing the reproduced trajectories generated by the control law \eqref{eq:clf-pol} with the training data, it becomes apparent that the inferred control policy reflects the depicted preferences in the human demonstrations.
This demonstrates that the learned Lyapunov function is flexible enough to encode the observed behavior in a similar fashion as a value function would. Furthermore, this is accomplished in a data efficient manner, since our method only requires $44$ human-generated data points for the presented results. \looseness=-1
\vspace{-0.025cm}
\section{Conclusion} \label{sec6}
This paper presents a novel perspective on the IRL inference problem, by using tools from dynamical system theory. Based on the concept of inverse optimality, the inference task is reformulated to learning CLFs from demonstrations. Through kernel-based non-parametric regression the CLFs are learned efficiently and with minimum structural assumptions. Finally, derived policies are certifiably stable and are shown to encode the preferences of the human demonstrator.

\addtolength{\textheight}{-12cm}   




\bibliographystyle{IEEEtran}
\bibliography{myBib}

\begin{thebibliography}{10}
\providecommand{\url}[1]{#1}
\csname url@samestyle\endcsname
\providecommand{\newblock}{\relax}
\providecommand{\bibinfo}[2]{#2}
\providecommand{\BIBentrySTDinterwordspacing}{\spaceskip=0pt\relax}
\providecommand{\BIBentryALTinterwordstretchfactor}{4}
\providecommand{\BIBentryALTinterwordspacing}{\spaceskip=\fontdimen2\font plus
\BIBentryALTinterwordstretchfactor\fontdimen3\font minus
  \fontdimen4\font\relax}
\providecommand{\BIBforeignlanguage}[2]{{%
\expandafter\ifx\csname l@#1\endcsname\relax
\typeout{** WARNING: IEEEtran.bst: No hyphenation pattern has been}%
\typeout{** loaded for the language `#1'. Using the pattern for}%
\typeout{** the default language instead.}%
\else
\language=\csname l@#1\endcsname
\fi
#2}}
\providecommand{\BIBdecl}{\relax}
\BIBdecl

\bibitem{Hussein17}
A.~Hussein, M.~Gaber, E.~Elyan, and C.~Jayne, ``Imitation learning: A survey of
  learning methods,'' \emph{ACM Comput. Surv.}, vol.~50, no.~2, 2017.

\bibitem{Abbeel04}
P.~Abbeel and A.~Ng, ``Apprenticeship learning via inverse reinforcement
  learning,'' in \emph{Proceedings of the International Conference on Machine
  Learning}, New York, NY, USA, 2004, p.~1.

\bibitem{Ziebart08}
B.~Ziebart, A.~Maas, J.~Bagnell, and A.~Dey, ``Maximum entropy inverse
  reinforcement learning,'' in \emph{Proceedings of the AAAI Conference on
  Artificial Intelligence}, 2008, pp. 1433--1438.

\bibitem{Finn16}
C.~Finn, S.~Levine, and P.~Abbeel, ``Guided cost learning: Deep inverse optimal
  control via policy optimization,'' in \emph{Proceedings of the International
  Conference on Machine Learning}, 2016, pp. 49--58.

\bibitem{Ijspeert13}
A.~{Ijspeert}, J.~{Nakanishi}, H.~{Hoffmann}, P.~{Pastor}, and S.~{Schaal},
  ``Dynamical movement primitives: Learning attractor models for motor
  behaviors,'' \emph{Neural Comput.}, vol.~25, no.~2, pp. 328--373, 2013.

\bibitem{Khansari-Zadeh2014}
S.~Khansari-Zadeh and A.~Billard, ``{Learning control Lyapunov function to
  ensure stability of dynamical system-based robot reaching motions},''
  \emph{Rob. Auton. Syst.}, vol.~62, no.~6, pp. 752--765, 2014.

\bibitem{Li2013}
Y.~Li, W.~Zhang, and X.~Liu, ``{Stability of Nonlinear Stochastic Discrete-Time
  Systems},'' \emph{J. Appl. Math.}, vol. 2013, no.~2, 2013.

\bibitem{Khalil2002}
H.~Khalil, \emph{{Nonlinear Systems}}, 3rd~ed.\hskip 1em plus 0.5em minus
  0.4em\relax Upper Saddle River, NJ: Prentice-Hall, 2002.

\bibitem{Freeman96}
R.~Freeman and P.~Kokotovic, ``\BIBforeignlanguage{English (US)}{Inverse
  optimality in robust stabilization},'' \emph{\BIBforeignlanguage{English
  (US)}{SIAM J Control Optim.}}, vol.~34, no.~4, pp. 1365--1391, 1996.

\bibitem{Teel2014}
A.~Teel, J.~Hespanha, and A.~Subbaraman, ``{A Converse Lyapunov Theorem and
  Robustness for Asymptotic Stability in Probability},'' \emph{IEEE Trans.
  Automat. Contr.}, vol.~59, no.~9, pp. 2426--2441, 2014.

\bibitem{Byrnes1993}
C.~Byrnes, W.~Lin, and B.~Ghosh, ``{Stabilization of discrete-time nonlinear
  systems by smooth state feedback},'' \emph{Syst. Control. Lett.}, vol.~21,
  no.~3, pp. 255--263, 1993.

\bibitem{Bacciotti2001}
A.~Bacciotti and A.~Biglio, ``{Some remarks about stability of nonlinear
  discrete-time control systems},'' \emph{Nonlinear Differ. Equ. Appl.},
  vol.~8, no.~4, pp. 425--438, 2001.

\bibitem{Hofmann2008}
T.~Hofmann, B.~Sch{\"{o}}lkopf, and A.~Smola, ``{Kernel Methods in Machine
  Learning},'' \emph{Ann. Stat.}, vol.~36, no.~3, pp. 1171--1220, 2008.

\bibitem{Micchelli2006}
C.~Micchelli, Y.~Xu, and H.~Zhang, ``{Universal Kernels},'' \emph{J. Mach.
  Learn. Res.}, vol.~7, pp. 2651--2667, 2006.

\bibitem{Rasmussen2006}
C.~Rasmussen and C.~Williams, \emph{{Gaussian Processes for Machine
  Learning}}.\hskip 1em plus 0.5em minus 0.4em\relax Cambridge, MA: The MIT
  Press, 2006.

\bibitem{Umlauft2017a}
J.~Umlauft, A.~Lederer, and S.~Hirche, ``{Learning Stable Gaussian Process
  State Space Models},'' in \emph{Proceedings of the American Control
  Conference}, 2017, pp. 1499--1504.

\bibitem{Scholkopf2001}
B.~Sch{\"{o}}lkopf, R.~Herbrich, and A.~Smola, ``{A Generalized Representer
  Theorem},'' in \emph{Proceedings of the 14th Annual Conference on
  Computational Learning Theory and and 5th European Conference on
  Computational Learning Theory}, 2001, pp. 416--426.

\bibitem{Berkenkamp2016a}
F.~Berkenkamp, R.~Moriconi, A.~Schoellig, and A.~Krause, ``{Safe Learning of
  Regions of Attraction for Uncertain, Nonlinear Systems with Gaussian
  Processes},'' in \emph{Proceedings of the IEEE Conference on Decision and
  Control}, 2016, pp. 4661--4666.

\bibitem{Lederer2019b}
A.~Lederer and S.~Hirche, ``{Local Asymptotic Stability Analysis and Region of
  Attraction Estimation with Gaussian Processes},'' in \emph{Proceedings of the
  IEEE Conference on Decision and Control}, 2019.

\bibitem{Lederer2020d}
A.~Lederer, M.~Kessler, and S.~Hirche, ``{GP3 : A Sampling-based Analysis
  Framework for Gaussian Processes},'' in \emph{Proceedings of the 21st IFAC
  World Congress}, 2020.

\end{thebibliography}

\end{document}